\documentclass[11pt]{llncs}
\usepackage[margin=1in]{geometry}

\usepackage{hyperref}

\usepackage{amssymb, amsfonts, amsmath}
\usepackage [autostyle, english = american]{csquotes}
\usepackage{tabularx}
\usepackage{graphicx}
\usepackage{url}
\usepackage{bbm}
\usepackage{algorithm}
\usepackage{algorithmic}

\usepackage[T1]{fontenc}
\usepackage[cal=boondox]{mathalfa}

\MakeOuterQuote{"}

\newcommand{\p}{\mathcal{p}}
\newcommand{\e}{\mathcal{e}}
\renewcommand{\i}{\mathcal{i}}
\renewcommand{\l}{\ell}

\newcommand{\W}{\mathcal{W}}

\newcommand{\size}[1]{\lvert #1 \rvert}

\renewcommand{\eqref}[1]{Equation~\ref{eq:#1}}
\newcommand{\eqsref}[2]{Equations~\ref{eq:#1} and~\ref{eq:#2}}

\newcommand{\secref}[1]{Section~\ref{sec:#1}}
\newcommand{\appref}[1]{Appendix~\ref{app:#1}}
\newcommand{\thmref}[1]{Theorem~\ref{thm:#1}}

\newcommand{\lemref}[1]{Lemma~\ref{lem:#1}}
\newcommand{\lemstworef}[2]{Lemmas~\ref{lem:#1} and~\ref{lem:#2}}
\newcommand{\lemsref}[3]{Lemmas~\ref{lem:#1}, ~\ref{lem:#2}, and~\ref{lem:#3}}

\newcommand{\defref}[1]{Definition~\ref{def:#1}}
\newcommand{\defsref}[3]{Definitions~\ref{def:#1}, ~\ref{def:#2}, and~\ref{def:#3}}
\newcommand{\corref}[1]{Corollary~\ref{cor:#1}}

\newcommand{\algref}[1]{Algorithm~\ref{alg:#1}}

\newcommand{\argmax}[1]{\underset{#1}{\text{argmax }}}
\newcommand{\mmin}[1]{\underset{#1}{\text{min }}}

\renewcommand{\inf}[1]{\underset{#1}{\text{inf }}}

\newcommand\ose{\textsc{ose }}
\newcommand\osen{\textsc{ose}}
\newcommand\NP{\textbf{NP}}

\newcommand{\BigO}[1]{\ensuremath{\operatorname{O}\!\left(#1\right)}}
\newcommand{\LittleO}[1]{\ensuremath{\operatorname{o}\!\left(#1\right)}}
\newcommand{\BigOmega}[1]{\ensuremath{\operatorname{\Omega}\!\left(#1\right)}}
\newcommand{\BigTheta}[1]{\ensuremath{\operatorname{\Theta}\!\left(#1\right)}}

\usepackage{natbib}
\newcommand\newcite{\citet}
\renewcommand\cite{\citep}
\begin{document}

\title{An Approximation Algorithm for Optimal Subarchitecture Extraction}

\author{Adrian de Wynter\inst{1}}%
\authorrunning{A. de Wynter}

\institute{$^1$Amazon Alexa, 300 Pine St., Seattle, Washington, USA 98101 \\
  \email{dwynter@amazon.com}
}
\maketitle

\begin{abstract}\footnote{This paper has not been fully peer-reviewed.}%

We consider the problem of finding the set of architectural parameters for a chosen deep neural network which is optimal under three metrics: parameter size, inference speed, and error rate. 
In this paper we state the problem formally, and present an approximation algorithm 
that, for a large subset of instances behaves like an FPTAS with an approximation error of 
$\rho \leq \size{1- \epsilon}$, 
and that runs in 
$O(\size{\Xi} + \size{W^*_T}(1 +  \size{\Theta}\size{B}\size{\Xi}/({\epsilon\, s^{3/2})}))$ 
steps, where $\epsilon$ and $s$ are input parameters; $\size{B}$ is the batch size; $\size{W^*_T}$ denotes the cardinality of the largest weight set assignment; and $\size{\Xi}$ and $\size{\Theta}$ are the cardinalities of the candidate architecture and hyperparameter spaces, respectively. 
\end{abstract}

\section{Introduction}

Due to their success and versatility, as well as more available computing power, deep neural networks have become ubiquitous on multiple fields and applications. 
In an apparent correlation between performance versus size, \emph{Jurassic} neural networks--that is, deep neural networks with parameter sizes dwarfing their contemporaries--such as the the 17-billion parameter Turing-NLG \cite{TuringNLG} have also achieved remarkable, ground-breaking successes on a variety of tasks. 

However, the relatively short history of this field has shown that Jurassic networks tend to be replaced by smaller, higher-performing models. For instance, the record-shattering, then-massive, 138-million parameter VGG-16 by \newcite{VGG16}, topped the leaderboards on the ImageNet challenge. Yet, the following year it was displaced by the 25-million parameter ResNet by \newcite{ResNet}.\footnote{\url{http://image-net.org/challenges/LSVRC/2014/index}; accessed June $26^{\text{th}}$, 2020. Both models were eventually displaced by DPN \cite{DPNs}, with about half the number of parameters of VGG-16.}. Another example is the 173-billion parameter GPT-3 from \newcite{GPT3}, which was outperformed by the 360-million parameter RoBERTa-large from \newcite{RoBERTa} in the SuperGLUE benchmark leaderboard.\footnote{\url{https://super.gluebenchmark.com/leaderboard}, accessed June $26^{\text{th}}$, 2020.}


Indeed, many strong developments have come from networks that have less parameters, but are based off these larger models. 
While it is known that overparametrized models train easily \cite{Livni}, 
a recent result by \newcite{LotteryTicket} states that it is possible to remove several neurons for a trained deep neural network and still retain good, or even better, generalizability when compared to the original model. From a computational point of view, this suggests that, for a trained architecture, there exists a smaller version with a weight configuration that works just as well as its larger counterpart. 
Simple yet powerful algorithms, such as the one by \newcite{Rewinding} have already supplied an empirical confirmation of this hypothesis, by providing a procedure through which to obtain high-performing pruned versions of a variety of models.

Our problem, which we dub the \emph{optimal subarchitecture extraction} (OSE) problem, is motivated from these observations. 
Loosely speaking, the OSE problem is the task of selecting the best non-trainable parameters for a neural network, such that it is optimal in the sense that it minimizes parameter size, inference speed, and error rate. 

Formally, the OSE problem is interesting on its own right--knowing whether such optimal subarchitecture can (or cannot) be found efficiently has implications on the design and analysis of better algorithms for automated neural network design, which perhaps has an even higher impact on the environment than the training of Jurassic networks. 
More importantly, hardness results on a problem are not mere statements about the potential running time of solutions for it; they are statements on the intrinsic difficulty of the problem itself.

\subsection{Summary of Results}

In this paper, we present a formal characterization of the OSE problem, which we refer to as \osen, and prove that it is weakly \NP-hard. We then introduce an approximation algorithm to solve it, and prove its time and approximation bounds. In particular, we show that the algorithm in question 
under a large set of inputs attains an absolute error bound of $c \leq \epsilon - 1$, and runs on $O(\size{\Xi} + \size{W^*_T}(1 +  \size{\Theta}\size{B}\size{\Xi}/({\epsilon\, s^{3/2})}))$ steps, where $\epsilon$ and $s$ are input parameters; $\size{B}$ is the batch size; $\size{W^*_T}$ is the cardinality of the largest weight set assignment; and $\size{\Xi},\size{\Theta}$ are the cardinalities of the candidate architecture and hyperparameter spaces, respectively. 
These results apply for a class of networks that fullfills three conditions: the intermediate layers are considerably more expensive, both in terms of parameter size and operations required, than the input and output functions of the network; the corresponding optimization problem is $L$-Lipschitz smooth with bounded stochastic gradients; and the training procedure uses stochastic gradient descent as the optimizer. 
We refer to these assumptions collectively as the \emph{strong} $AB^nC$ \emph{property}, and note that it fits the contemporary paradigm of deep neural network training. 
We also show that if we assume the optimization to be $\mu$-strongly convex, this algorithm is an FPTAS with an approximation ratio of $\rho \leq \size{1- \epsilon}$; and remark that the results from the FPTAS version of the algorithm hold regardless of convexity, if we assume the optimal solution set to be the set of reachable weights by the optimizer under the input hyperparameters sets $\Theta$. 

\subsection{Related Work}\label{sec:relatedwork}

The OSE problem is tied to three subfields on machine learning: neural architecture search (NAS), weight pruning, and general neural network compression techniques. 
All of these areas are complex and vast; while we attempted to make an account of the most relevant works, it is likely we unknowingly omitted plenty of contributions. 
We link several surveys and encourage the reader to dive deeper on the subjects of interest. 

Given that OSE searches for optimal architectural variations, it is a special case of NAS. 
Indeed, \emph{True} NAS, where where neither the weights, architectural parameters, or even the functions composing the architecture are fully defined, has been explored in both an applied and a theoretical fashion. 
Several well-known examples of the former can be seen in \newcite{Zoph2016NeuralAS,AmoebaNet,liu2019darts}, and \newcite{GeometryAware}, when applied to deep learning, although early work on this area can be seen in \newcite{CarpenterAndGrossberg} and \newcite{SchafferCaruana}. The reader is referred to the survey by \newcite{Elsken2019NeuralAS} for NAS, and the book by \newcite{automl_book} for an overview of its applications in the nascent field of automated neural network design, or AutoML.

From a computational standpoint our problem can be seen as a special case of the so-called Architecture Search Problem from \newcite{deWynterFA}, where the architecture remains fixed, but the non-trainable parameters (e.g., the dimensionality of the hidden layers) for said architecture do not. 
The hardness results around OSE are not surprising, and are based off several important contributions to the complexity theory of training a neural network in the computational model: 
the loading problem was shown to be \NP-complete for a 3-node architecture with hard threshold \cite{BlumAndRivest} and continuous activation functions \cite{DasGupta}; 
\newcite{Arora} showed that a 2-layer neural network with a ReLU activation unit could learn a global minima in exponential time on the input dimensionality; 
and, more generally, the well-known results by \newcite{Klivans} and \newcite{DainelyHalfSpaces} prove that intersections of half-spaces are not efficiently learnable under certain assumptions. 
It is argued by \newcite{ShalevUML} that this result has implications on the impossibility of efficiently training neural networks; likewise, \newcite{AroraBabai} showed that surrogating a loss function and minimizing the gradient--a common practice in deep learning to circumvent non-convexity--is intractable for a large class of functions. 
Other works proving stronger versions of these results can be found in \newcite{DanielyAndShaiS}.

Our work, however, is more closely related to that by \newcite{Judd}, who showed that approximately loading a network, independently of the type of node used, is \NP-complete on the worst case. It presents a very similar approach insofar as both their approach and ours emphasize the architecture of the network. 
However, our problem differs in several aspects: namely, the definition of an architecture focuses on the combinatorial aspect of it, rather than on the functional aspect of every layer; we do not constrain ourselves to simply finding the highest-performing architecture with respect to accuracy, but also the smallest bitstring representation that also implies lower number of operations performed; and we assume that the input task is learnable by the given architecture--a problem shown to be hard by \newcite{DasGupta}, but also proven to be an insufficient characterization of the generalizability of deep neural networks \cite{Livni}. 

Indeed, when framed as a deep neural network multiobjective optimization problem, the OSE problem finds similes on the considerable work on one-size-fits-all algorithms designed to improve the performance of a model on these metrics. 
See, for example, \newcite{DeepCompression}. 
Ohter common techniques for neural network compression involve knowledge distillation \cite{bucilu2006model,hinton2015distilling}, although, by design, they do not offer the same guarantees as the OSE problem. 
On the other hand, a very clever method was recently presented by \newcite{Rewinding}, with much success on a variety of neural networks. 
We consider OSE to be a problem very similar to theirs, although the method presented there is a form of \emph{weight pruning}, where the parameters of a trained network are removed with the aim of improving either inference speed or parameter size. 
On the other hand, OSE is a "bottom up" problem where we begin with an untrained network, and we directly optimize over the search space. 

Finally, we would like to emphasize that the idea of a volumetric measure of quality for a multiobjective setting is well-known, and was inspired in our work by \newcite{IBMPaper} in the context of quantum chip optimization. 
Moreover, a thorough treatment of the conditions where more data and more parameters in the context of abyssal models is done in a recent paper by \newcite{Nakkiran2020Deep}, 
and a good analysis of the convergence rate of such optimizers in non-convex settings and under similar assumptions to those of the strong $AB^nC$ property can be found in \newcite{NonconvexStudy}.

\subsection{Outline}

The remaining of this paper is structured as follows: in \secref{main} we introduce notation, a formal definition of the OSE problem, and conclude by proving its hardness. 
We then begin in \secref{fptas} by stating and proving some assumptions and properties needed for our algorithm, and subsequently formulate such procedure. 
Time and approximation bounds for a variety of input situations are provided in \secref{analysis}. 
Finally, in \secref{conclusion}, we discuss our work, as well as its potential impact and directions for further research.

\section{Optimal Subarchitecture Extraction}\label{sec:main}

\subsection{Background}

We wish to find the set of non-trainable parameters for a deep neural network $f \colon \mathbb{R}^p \rightarrow \mathbb{R}^q$ such that its parameter size, inference speed, and error rate on a set $D$ are optimal amongst all such non-trainable parameters; in addition, the functions that compose $f$ must remain unchanged. 

Formally, we refer to a \emph{layer} as a nonlinear function  $l_i(x;W_i)$ that takes in an input $x \in \mathbb{R}^{p_i}$, and a finite set of trainable weights, or parameters, $W_i = \{w_{i,1}, \dots, w_{i,k}\}$, such that every $w_{i,j} \in W_i$ is an $r$-dimensional array $w_{i,j} \in \mathbb{R}^{d^{(1)}_{i,j} \times d^{(2)}_{i,j} \times \dots \times d^{(r)}_{i,j}}$ for some $d^{(1)}_{i,j}, d^{(2)}_{i,j}, \dots, d^{(r)}_{i,j}$, and returns an output $y \in \mathbb{R}^{q_i}$. 

This definition is intentionally abstracted out: in practice, such a component is implemented as a computer program not solely dependent on the dimensionality of the input, but also other, non-trainable parameters that effect a change on its output value. 
We can then parametrize said layer with an extra ordered tuple of variables $\xi_i = \langle d^{(1)}_{i,j}, d^{(2)}_{i,j}, \dots, d^{(r)}_{i,j} \rangle$, and explicitly rewrite the equation for a layer as $l_i(x; W_i; \xi_i)$. 
For the rest of the paper we will adopt this programmatic, rather than functional, view of a layer, and remark that this perspective change is needed for notational convenience and has little effect on the mathematics governing this problem. We will also assume, for simplicity, that the output of the last layer is binary, that is, $f \colon \mathbb{R}^p \rightarrow \{0,1\}$. 

With the definition of a layer, we can then write a neural network \emph{architecture} $f$ as a continuous, piecewise-linear function, composed of a sequence of $n$ layers: 
\begin{equation}\label{eq:architecturedef}
f(x; W; \xi) = l_n(l_{n-1}(\dots(l_1(x; W_1; \xi_1)\dots); W_{n-1}; \xi_{n-1}); W_n; \xi_n),
\end{equation}


and formally define the architectural parameters and search spaces:

\begin{definition}
The \emph{architectural parameter set} of a neural network architecture $f(x; W; \xi)$ is a finite ordered tuple $\xi = \langle \xi_1, \dots, \xi_n \rangle$ of variables such that $\xi_i \in \xi$ iff it is non-trainable.
\end{definition}

Intuitively, an architectural parameter is different from a trainable parameter because it is a variable which is assigned a value when the architecture is instantiated, likewise remaining unchanged for the lifetime of this function. A change on the value assignment of said variable is capable of effect change on the three functions we are optimizing over, even when utilizing the same training procedure. However, any assignment of $\xi_i$ can be mapped to a specific set of possible trainable parameter assignments. 

\begin{definition}
The \emph{search space} $\Xi = \{\xi^{(1)}, \dots \xi^{(m)}\}$ is a finite set of valid assignments for the architectural parameters of a neural network architecture $f(x; W; \xi^{(i)})$, such that for every assignment $\xi^{(i)}$, the corresponding weight assignment set $W^{(i)} \in W$ is non-empty.
\end{definition}

Even though the large majority of neural network components could be written as an affine transformation, for practical purposes we leave out the implementation specifics for each layer; the problem we will be dealing with operates on $\Xi$, and treats the architecture as a black box. Our only constraint, however, is that all trainable and architectural parameters in \eqref{architecturedef} must be necessary to compute the output for the network. 

The following example illustrates a "small" problem, focusing on a specific architecture and its variations:

\begin{example}
Let our input be $x \in \mathbb{R}^{p}$, and $H, J, A$ be positive integers. Consider the following architecture, composed (in sequence) of a Transformer \cite{Attention}, a linear layer, and a sigmoid activation unit $\sigma(\cdot)$:

\begin{equation}\label{eq:example}
f(x) = \sigma\left( W_L \text{softmax}\left( \frac{(W_Kx + b_K )(W_Qx + b_Q)^{\top}}{\sqrt{\frac{H}{A}}}(W_Vx + b_V)  \right)   + b_L \right)
\end{equation}
Here, $W_K, W_V, W_Q \in \mathbb{R}^{H \times p}, W_L \in \mathbb{R}^{J \times H}$, and $b_K, b_V, b_Q \in \mathbb{R}^{H}$, $b_L \in \mathbb{R}^{J}$ are trainable parameters. 
Although not explicitly stated in \eqref{example}, $A$ is constrained to be divisible by $H$; it specifies the step size (i.e., number of rows) for the argument over which the softmax$(\cdot)$ function is to be executed; and acts as a regularization parameter.\footnote{This parameter is referred to in the literature as the \emph{number of attention heads}.} 
Thus, the architectural parameters for this architecture are given by $\xi =\langle H, J, A \rangle$, and the search space would be a subset of the countably infinite number of assignments to $H,J$ and $A$ that would satisfy \eqref{example}. 
Finally, note how the Transformer unit in this example is considered to be its own layer, in spite of being a block of three linear layers, a division, and an activation function. 
\end{example}

For notational completeness, in the following we revisit or define the \emph{parameter size}, \emph{inference speed}, and \emph{error} of a network. 

\begin{definition}\label{def:paramfunction}
The \emph{parameter size} $\p(f(\cdot; W; \xi))$ of a network $f(\cdot; W; \xi)$ is defined as the number of trainable parameters in the network:

\begin{equation}\label{eq:parameq}
\p(f(\cdot; W; \xi)) = \sum_{\{l_i(\cdot; W_i; \xi_i) \colon l_i \in f\}} \size{W_i}
\end{equation}
\end{definition}

The parameter size of a network is different than the size it occupies in memory, since compressed models will require less bits of storage space, but will still have to encode the same number of parameters. In other words, \defref{paramfunction} provides a bound on the length of the bitstring representation of $f$. 

\begin{definition}\label{def:inffunction}
The \emph{inference speed} $\i(f(\cdot; W; \xi))$ of a network $f(\cdot; W; \xi)$ is defined as the total number of steps required to compute an output during the forward propagation step for a fixed-size input.
\end{definition}

Similar to the parameter size, \defref{inffunction} provides a lower bound on the requirements to compute an output, barring any implementation-specific optimizations.
We must also point out that, outside of the model of computation used, some compression schemes (e.g., floating point compression) also have a direct impact on the actual inference speed of the network, roughly proportional to its compression ratio. We do not consider compressed models further, but we note that such enhancements can be trivially included into these objective functions without any impact to the complexity bounds presented in this paper. 

\begin{definition}\label{def:errorfunction}
Let $D$ be a set such that it is sampled i.i.d. from an unknown probability distribution; that is, $D = \{\langle x_i, y_i \rangle\}_{i=1,\dots,m}$ for $x_i \in \mathbb{R}^p$ and $y_i \in \{0,1\}$. The \emph{error rate} of a network $f(\cdot; W; \xi)$ with respect to $D$ is defined as:

\begin{equation}\label{eq:errorfunction}
\e(f(\cdot; W; \xi), D) = \frac{1}{\vert D \vert} \sum_{\langle x_i, y_i\rangle \in D} \mathbbm{1}[f(x_i; W; \xi) \neq y_i]
\end{equation}

where we assume that Boolean assignments evaluate to $\{0,1\}$.

\end{definition}

The output of the function $\e(f(\cdot; W; \xi), D)$ is directly dependent on the trained parameters, as well as (by extension) the hyperparameters and training procedure used to obtain them. In particular, we note that the \emph{optimal weight set} $W^*$ is the one that minimizes the error across all possible assignments in $W$.  
Albeit obvious, this insight will allow us to prove the complexity bounds for this problem, as well as provide an approximation algorithm for it. 
It is then clear that we are unable to pick an optimal architecture set $\xi^*$ without evaluating the error of all the candidate networks $f(\cdot; W^{*}; \xi^{(j)})$, $\forall \xi^{(j)}, \xi^* \in \Xi$ and $W^{*} \in W$. 
The main results of our work, presented in \secref{fptas}, focus on obtaining a way to circumvent such a limitation. 

Whenever there is no room for ambiguity regarding the arguments, we will write \defsref{paramfunction}{inffunction}{errorfunction}, as $\p(f)$,$\i(f)$, and $\e(f)$, respectively, and obviate the input parameters to the architecture; moreover, we will refer to $D$ as the \emph{dataset}, following the definitions above.

\subsection{Optimal Subarchitecture Extraction}

We are now ready to formally define the optimal subarchitecture extraction problem:

\begin{definition}[\osen]
Given a dataset $D = \{\langle x_i, y_i \rangle\}_{i=1,\dots,m}$ sampled i.i.d. from an unknown probability distribution, search space $\Xi$, a set of possible weight assignments $W$, a set of hyperparameter combinations $\Theta$, and an architecture $f(x) = l_n(\dots l_1(x; W_1;\xi_1)\dots); W_n; \xi_n)$, 
find a valid assignment $\xi^* \in \Xi$, $W^* \in W$ such that $\p(f(\cdot; W^*; \xi^*))$, $\i(f(\cdot; W^*; \xi^*))$ and $\e(f(\cdot; W^*; \xi^*), D)$ are minimal across all $\xi^{(i)} \in \Xi$, $W^{(i)} \in W$, and $\theta^{(i)} \in \Theta$.
\end{definition}

We assume that $\Xi$ is \emph{well-posed}; that is, $\Xi$ is a finite set for which every assignment $\xi^{(k)} \in \Xi$ of architectural parameters to $f(\cdot; \cdot; \xi^{(k)})$ is valid. 
In other words, $\xi^{(k)} \in \Xi$ is compatible with every layer $l_j(\cdot; \cdot; \xi_j^{(k)}) \in f(\cdot; \cdot; \xi^{(k)})$ in the following manner:
\begin{itemize}
	\item The dimensionality of the first operation on $l_1(\cdot; \cdot; \xi_1^{(k)})$ is compatible with the dimensionality of $x_i \in \mathbb{R}^p$, for $\langle x_i, y_i \rangle \in D$.
	\item The dimensionality of the last operation on $l_1(\cdot; \cdot; \xi_n^{(k)})$ is compatible with the dimensionality of all $y_i \in \mathbb{R}^q$ for $\langle x_i, y_i \rangle \in D$.
	\item For any two subsequent operations, the dimensionality of the first operation is compatible with the dimensionality of the ensuing operation.
\end{itemize}

Note that the well-posedness of $\Xi$ can be achieved in a programmatic manner and in linear time so long as $p$ and $q$ are guaranteed to be constant. Likewise, it does not need to encode solely dimensionality parameters. Mathematically, we expect $\Xi$ to be the domain of a bijective function between the search space and all possible valid architectures. 


\subsection{Complexity of \ose}\label{sec:complexityose}

We prove in \thmref{fptasosethm} that \ose is weakly \NP-hard. 
To begin, we formulate the decision version of \osen, \textsc{ose-dec}, as follows:

\begin{definition}[\textsc{ose-dec}]\label{def:osedec}
Given three numbers $k_p, k_i, k_e$, and a tuple for \ose $\langle f, D, W, \Xi, \Theta \rangle$, is there an assignment

\begin{equation}
\xi^* \in \Xi, W^* \in W
\end{equation}

for $f(\cdot; W; \xi)$ such that 
\begin{align}
\p(f(\cdot; W^*; \xi^*)) &\leq k_p, \\
\i(f(\cdot; W^*; \xi^*)) &\leq k_i, \text{and} \\
\e(f(\cdot; W^*; \xi^*), D) &\leq k_e\text{?}
\end{align}
\end{definition}

For the following, we will make the assumption, without loss of generality,\footnote{Computationally speaking, hypothesis spaces are defined in finite-precision machines, and hence are finite.} that all inputs to \defref{osedec} are integer-valued. 

\begin{lemma}\label{lem:npcose}
\textsc{ose-dec} $\in$ \NP-hard.
\end{lemma}
\begin{proof}

In \appref{npcoseproof}. 

\end{proof}


\begin{theorem}\label{thm:fptasosethm}
Let $\bar{f}$ be the size of a network $f$ in bits. 
Assume that, for any instance of \textsc{ose-dec}, $\i(f) \in \BigOmega{\text{poly}(\bar{f})}$ for all assignments $\xi^{(i)} \in \Xi$. Then \textsc{ose-dec} is weakly \NP-hard.
\end{theorem}
\begin{proof}
By construction. 
Let $I = \langle f, D, W, \Xi, \Theta, k_p, k_i, k_e \rangle$ be an instance of \textsc{ose-dec}. 

Remark that, by definition, the parameter size of a neural network increases with the cardinality of $W^{(i)}$, which is itself dependent on the choice of $\xi^{(i)}$. 
Note also that the inference speed $\i(f(\cdot; W^{(i)}; \xi^{(i)}))$ can be obtained for free at the same time as computing the error, as this function runs in $\Theta(\i(f(\cdot; W^{(i)}; \xi^{(i)}))\size{D})$ steps, for any $f(\cdot; W^{(i)}; \xi^{(i)})$. 

Consider now the following exact algorithm:
\begin{itemize} 
	\item Construct tables $T_1,\dots,T_{k_p}$, each of size $\size{\Theta}\size{W^{(i)}}$, for every $W^{(i)} \in W$, $i < k_p$. 
	\item For every table $T_i$ compute $\e(f(\cdot; W^{(i), j}; \xi^{(i)}), D)$ for every assignment $W^{(i), j} \in W^{(i)}$ and hyperparameter set $\theta^{(i)} \in \Theta$. If the condition $\i(f) < k_i$ and $\e(f) < k_e$ holds, stop and return "yes". 
	\item Otherwise, return "no".
\end{itemize}

It is important to highlight that the algorithm does not perform any training, and so the values for some entries of $\theta^{(i)} \in \Theta$ might be unused. 
That being said, given the ambiguity surrounding the definition of $\Theta$, we can bound the runtime of this procedure to
\begin{equation}\label{eq:osedecnaiveruntime1}
 \BigO{\size{\Theta}\size{D}\left(\sum_i^{k_p}\size{W^{(i)}}\i(f(\cdot; \cdot; \xi^{(i)}))\right)},
\end{equation}

steps. 


Note that, although unnervingly slow, \eqref{osedecnaiveruntime1} implies this algorithm is polynomial on the cardinalities of $\Theta$, $W^{(i)}$, and $D$, rather than on their magnitudes. 

However, this algorithm actually runs in pseudopolynomial time, since the instance's runtime has a dependency on $\i(f)$. 
From the theorem statement, we can assume that $\i(f)$ is lower-bounded by the size in bits of $f(\cdot; W^{(i), j}; \xi^{(i)})$, and, since $\size{W^{(i)}} \in \BigTheta{\size{f(\cdot; W^{(i), j}; \xi^{(i)})}}$ in both length and magnitude, it is also lower-bounded by the values assigned to the elements of $ W^{(i), j} 
\in W^{(i)}$. 

This in turn implies that, if $W^{(i), j}_k \in W^{(i)}$ is chosen to be arbitrarily large, the runtime of our algorithm is not guaranteed to be polynomial. 
It follows that \textsc{ose-dec} is weakly $\NP$-hard.

 
\end{proof}

\begin{remark}
The assumption from \thmref{fptasosethm} regarding the boundedness of $\i(f)$ is not a strong assumption. 
If it were not to hold in practice, we would be unable to find tractable algorithms to evaluate the error $\e(f)$. 
\end{remark}

\begin{corollary}\label{cor:polytimesolns}
Let $I = \langle f, D, W, \Xi, \Theta \rangle$ be an instance of \ose such that, $\forall \xi^{(i)}, \xi^{(j)} \in \Xi$, and any assignments of $W^{(i)}, W^{(j)} \in W$,

\begin{equation}
\e(f(\cdot; W^{(i)}; \xi^{(i)}), D) = \e(f(\cdot; W^{(j)}; \xi^{(j)}), D). 
\end{equation}

This instance admits a polynomial time algorithm. 
\end{corollary}
\begin{proof}
Note that $\p(f)$ and $\i(f)$ can be computed in time linear on the size of the architecture by employing a counting argument. 
Therefore this instance of \ose is equivalent to the single-source shortest-paths problem between a fixed source and target $s,t$. To achieve this, construct a graph with a vertex per every architectural parameter $\xi^{(i)}_{j} \in \xi^{(i)}$, for all $\xi^{(i)} \in \Xi$. Assign a zero-weight edge in between every $\xi^{(i)}_{j}, \xi^{(i)}_{j+1} \in \xi^{(i)}$, and between $s$ and every $\xi^{(i)}_{1}$. Finally, add an edge of weight $\p(f(\cdot; \cdot; \xi^{(i)})) + \i(f(\cdot; \cdot; \xi^{(i)}))$ between every $\xi^{(i)}_{-1}$ and $t$. This is well-known to be solvable in polynomial time.
\end{proof}

It is important to highlight that the hardness results from this section are not necessarily limiting in practice (cf. \thmref{fptasosethm}), and worst-case analysis only applies to the general case. 
Current algorithms devised to train deep neural networks rely on the statistical model of learning, and routinely achieve excellent results by approximating a stationary point corresponding to this problem in a tractable number of steps. 
Nonetheless, computationally speaking, \thmref{fptasosethm} highlights a key observation of the OSE problem: the problem could be solved (i.e., brute-forced) for small instances efficiently, but the compute power required for larger inputs quickly becomes intractable.

\section{An Approximation Algorithm for \ose}\label{sec:fptas}

In this section we introduce our approximation algorithm for \osen. 
Our strategy is simple: we rely on surrogate functions for each of our objective functions, as well as a scalarization of said surrogates which we refer to as the $W$-\emph{coefficient}. 
We begin by describing these functions, and then we show that a solution optimal with respect to the $W$-coefficient is an optimal solution for \osen. 
We conclude this section by introducing the algorithm.

\subsection{Surrogate Functions}

We surrogate our objective functions such that they are all in terms of the variable set for $\Xi$, with the ultimate goal of transforming \ose into a volume maximization problem. 
Assuming that $\p(f)$ and $\i(f)$ are expressable in terms of $\Xi$ is not a strong assumption: if such functions were not at least weakly correlated with $\Xi$, we could assume their values to be as constants for any $f(\cdot; \cdot; \xi)$ in $\Xi$, and solve this as a neural network training problem. 
Seen in another way, as the number of trainable parameters increases, we can expect $\p(f)$ and $\i(f)$ to likewise increase, although not necessarily in a linear fashion. 
On the other hand, $\e(f)$ is a function directly affected by the values of the \emph{trained} weights $W$, and not by the architectural parameters. This means that the expressability of $\e(f)$ as a function of $\Xi$ requires stronger assumptions than for the other two metrics, and directly influences our ability to approximate a solution for \osen. 

For the following, let $F$ be the set of possible architectures, such that it is the codomain of some bijective function $A$ assigning architectural parameters from $\Xi$ to specific \emph{candidate} architectures,

\begin{equation}\label{eq:axif}
A: \Xi \rightarrow F.
\end{equation}

Additionally, we will use the shorthand poly$(\Xi)$ to denote a polynomial on the variable set $\{\xi_{i,1}, \dots \xi_{i,m}\colon\xi^{(i)} = \langle \xi_{i,1}, \dots \xi_{i,m}\rangle \in \Xi\}$, that is,

\begin{equation}
\text{poly}(\Xi) = \text{poly}(\xi_{1,1}, \dots , \xi_{i,1}, \dots \xi_{i,m}, \dots, \xi_{n, m}). 
\end{equation}

\subsubsection{Surrogate Inference Speed}

We will assign constant-time cost to add-multiply operations, and linear on the size of the input for every other operation.\footnote{In the deep learning literature, the surrogate inference speed is sometimes measured via the number of \emph{floating point operations}, or FLOPs--that is, the number of add-multiply operations. See, for example, \newcite{FLOPS}. However, this definition would imply that $\hat{\i}(f(\cdot; W; \xi)) \in \BigO{\i(f(\cdot; W; \xi))}$, which is undesirable from a correctness point of view.} We assume that the computations are taken up to some finite precision $b$. 

\begin{multline}\label{eq:ihateq}
\hat{\i}(f(\cdot; W; \xi)) = \sum_{l_i(\cdot; W_i; \xi_i) \in f}\text{number of additions in $l_i(\cdot; W_i; \xi_i)$} +\\ 
\sum_{l_i(\cdot; W_i; \xi_i) \in f}\text{number of multiplications in $l_i(\cdot;W_i; \xi_i)$} + \\
\sum_{j \in [1,\dots, b]}\sum_{l_i(\cdot; W_i; \xi_i) \in f} j \cdot \text{number of other operations in $l_i(\cdot;W_i; \xi_i)$ of length $j$}.
\end{multline}

\begin{lemma}\label{lem:ihatxi}
For all $f \in F$,

\begin{equation}
\hat{\i}(f(\cdot; W; \xi)) \in \BigO{\text{poly}(\Xi)}.
\end{equation}

\end{lemma}
\begin{proof}
Recall that $k$-ary mathematical operations performed on an set of inputs will polynomially depend on the size of each of its members, which is given by $\tilde{O}(n^{\omega})$ for $\omega$ being a constant that often depends on the chosen multiplication algorithm. 
Given that the size of every member is specified by the length of the bitwise representation of every element, times the number of elements on said member--which is by definition, encoded in $\xi$--we can have a loose bound on the number of operations specified by every positional variable $\xi_{i} \in \xi$, for every $\xi \in \Xi$. 
By the compositional closure of polynomials, this value remains polynomial on the variable set for $\Xi$. 
\end{proof}

\begin{lemma}\label{lem:ihat}
For any $f \in F$,
\begin{equation}
\i(f(\cdot; W; \xi)) \in \BigO{\hat{\i}(f(\cdot; W; \xi)}.
\end{equation}
\end{lemma}
\begin{proof}
Immediate from \lemref{ihatxi}. 

\end{proof}

\lemref{ihatxi} implies that we did not need to specify linear cost for all other operations. 
In fact, if we were to use their actual cost with respect to their bitstring size and the chosen multiplication algorithm, we could still maintain polynomiality, and the bound from \lemref{ihat} would be tight. 
However, \eqref{ihateq} suffices for the purposes of our analysis, since this implies that the surrogate inference speed of a network is bounded by its size.

It is important to note that \eqref{ihateq} is not an entirely accurate depiction of how fast can an architecture compute an output in practice. Multiple factors, ranging from hardware optimizations and other compile-time techniques, to even which processes are running in parallel on the machine, can affect the total wall clock time between computations. 
Nonetheless, \eqref{ihateq} provides a quantification of the approximate "cost" inherent to an architecture; regardless of the extraneous factors involved, the number of operations performed will remain constant \cite{MatrixBook}, and it is known that all of these computations remain polynomial regardless of the model of computation utilized. 
For the purposes of our algorithm, however, we will limit ourselves to use this quantity in relative terms: that is, we shall be more interested in the comparison between $\hat{\i}(f(\cdot; W^{(k)}; \xi^{(i)}))$ and $\hat{\i}(f(\cdot; W^{(l)}; \xi^{(j)}))$ for two $\xi^{(i)}, \xi^{(j)} \in \Xi$.

\subsubsection{Surrogate Parameter Size}
We do not surrogate $\p(f)$, as it is already expressible as a polynomial on $\Xi$:

\begin{lemma}\label{lem:phat}
For any $f \in F$,
\begin{equation}
\p(f(\cdot; W; \xi)) \in \BigTheta{\text{poly}(\Xi)}.
\end{equation}
\end{lemma}
\begin{proof}
By \defref{paramfunction}, note that $\p\colon F \rightarrow \mathbb{N}_{\geq 0}$ is the function taking in an architecture and returning its parameter size. 
Let $\hat{\p}\colon\Xi \rightarrow \mathbb{N}_{\geq 0}$ be the function taking in the architectural parameter set of the architecture, and returning its parameter size. $\hat{\p}$ is the pullback of $A$ by $\p$. 
Polynomiality follows from the fact that every $r$-dimensional set of weights $w_{i,j} \in \mathbb{R}^{d^{(1)}_{i,j} \times d^{(2)}_{i,j} \times \dots \times d^{(r)}_{i,j}}$ belonging to some layer $l_i(\cdot; W_i; \xi_i)$, is parametrized in $\Xi$-space by $r$ corresponding architectural parameters $\xi_{i,k} = d^{(k)}_{i,j}$. This will contribute a total of $\prod_k^r d^{(k)}_{i,j}$ trainable parameters to the evaluation of $\p(f)$. 
On the other hand, the non-trainable architectural parameters that do not correspond to a dimensionality will not affect the value of $\p(f)$. 
\end{proof}

\subsubsection{Surrogate Error}

It was mentioned in \secref{complexityose} that the error is hard to compute directly,
and hence we must rely on the surrogating of $\e(f)$ to an empirical loss function. 
Following standard optimization techniques we then define the \emph{surrogate error} as the empirical loss function over a subset $B \subset D$:

\begin{equation}\label{eq:ehateq}
\hat{\e}(f(\cdot; W; \xi), B) = 
\frac{1}{\vert B \vert} \sum_{\langle x_i, y_i\rangle \in B \subset D} \l(f(x_i; W; \xi), y_i).
\end{equation}

for some smooth function $\l\colon\mathbb{R}^p \times \mathbb{R}^q \rightarrow [0,1]$ that upper bounds \defref{errorfunction} and is classification-calibrated.\footnote{See, for example, \newcite{BartlettAndJordan} and \newcite{NguyenWainwrightJordan}. } 
This function does not need to be symmetric; but in \secref{analysis} we will impose further constraints on it. It is common to minimize the empirical loss by the use of an optimizer such as stochastic gradient descent (SGD), although in deep learning issues arise due to its non-convexity. 

Our goal is then to show that such a loss function can be approximated as a function of $\Xi$. 

\begin{lemma}\label{lem:epoly}
Let $I = \langle f, D, W, \Xi, \Theta, \l \rangle$ be an instance of \osen, such that the same training procedure and optimizer are called for every $f \in F$. Fix a number of iterations, $s$, for this training procedure. Then, for any $f \in F$, with fixed $B \subset D$ and $\theta \in \Theta$, 
$\hat{\e}(f(\cdot; W; \xi), B)$ corresponds to the image of some polynomial $\text{poly}_{B, s, \theta}(\Xi)$ evaluated at $f$. 

\end{lemma}
\begin{proof}
Note that a training procedure has, at iteration $s$, exactly one weight set assigned to any $f \in F$. We can then construct a function that assigns a quality measure to every architecture on $f$ based on $\hat{\e}(f)$ at $s$, that is, $\mathcal{f}\colon F \rightarrow [0,1]$. 
Polynomiality follows from the fact that $\mathcal{f}$ is the pullback of $\hat{\e}$ by $\p$; by \lemref{phat}, $\mathcal{f}$ is then parametrized by a polynomial on $\Xi$.
\end{proof}

\begin{lemma}\label{lem:ehat}
For any $f \in F$ and an appropriate choice of $\l(\cdot, \cdot)$,
\begin{equation}
\e(f(\cdot; W; \xi), D) \in \BigO{\hat{\e}(f(\cdot; W; \xi), B)} .
\end{equation}
\end{lemma}
\begin{proof}
By construction, $\l(\cdot, \cdot)$ is designed to upper bound $\e(f)$. 
\end{proof}

As before, we will employ the shorthands $\hat{\e}(f)$ and $\hat{\i}(f)$ for \eqsref{ehateq}{ihateq}. 
It is clear that any solution strategy to \ose will need, from the definition of the surrogate error, to take an extra input $\l$. We will account for that and denote such an instance as $I = \langle f, D, W, \Xi, \Theta, \l \rangle$.

\subsection{The $W$-coefficient}

The final tool we require in order to describe our algorithm is what we refer to as the $W$-coefficient,\footnote{
Visually, the $W$-coefficient is proportional to the volume of the $\Xi$-dimensional conic section by the volume spanned by $\langle \p(T) - \p(f), {\i}(T) - \hat{\i}(f), \hat{\e}(f) \rangle$ in $\Xi$-space, hence its name.} which is the scalarization of our objective functions. 
We begin by definining a crucial tool for our algorithm, the \emph{maximum point on} $\Xi$:

\begin{definition}[The maximum point on $\Xi$]\label{def:maxpoint}
Let $T(\cdot; W_T^*; \xi_T)$ be an architecture in $F$ such that $\forall f \in F$, $\p(f) < \p(T)$ and $\i(f) < \i(T)$. 
We refer to $T$ as the \emph{maximum point} on $\Xi$. 
\end{definition}

Note that the maximum point on $\Xi$ is not an optimal point--our goal is to minimize all three functions, and $T(\cdot; W_T^*; \xi_T)$ does not have a known error rate on $D$. 
By this reason, the parameters corresponding to the maximum point cannot be seen as a nadir objective vector. 

\begin{definition}\label{def:wcoeff}
Let $f(\cdot; W; \xi)$, $T(\cdot; W^*_T; \xi_T) \in F$ be two architectures, such that $T$ is a maximum point on $\Xi$. The $W$-coefficient of $f$ and $T$ is given by: 

\begin{equation}\label{eq:wcoeffeq}
\W(f, T) = 
\left(  \frac{\p(T) - \p(f)}{\p(T)} \right)
\left(  \frac{\hat{\i}(T) - \hat{\i}(f)}{\hat{\i}(T)} \right)
\left(
\frac{1}{
 \hat{\e}(f)
}\right).
\end{equation}
\end{definition}

The need for the normalization terms in \eqref{wcoeffeq} becomes clear as we note that the practical range of the functions might differ considerably. 
We now prove that the $W$-coefficient can be seen as the scalarization of our surrogate objective functions $\p(f)$, $\hat{\e}(f)$, and $\hat{\i}(f)$: 

\begin{lemma}\label{lem:wcoefficientproof}
Let $T(\cdot; W^*_T; \xi_T) \in F$ be a maximum point on $\Xi$. Let $OPT$ be the Pareto-optimal set for the multiobjective optimization problem:

\begin{align}
\text{Minimize }  &\p(f), \hat{\e}(f), \hat{\i}(f) \nonumber \\
\text{Subject to }&D, \Theta, \text{and } f \in F 
\end{align} 
Then $f^*(\cdot; W^*; \xi^*)  \in OPT$ if and only if $\W(f^*, T) > \W(f, T)$ $\forall f,f^* \in F$. 
\end{lemma}
\begin{proof}
"If" direction: by \eqref{wcoeffeq} and \defref{maxpoint}, the terms corresponding to $\p(T) - \p(f)$  (resp. $\hat{\i}(T) - \hat{\i}(f)$) will be greater whenever $\p(f)$ (resp. $\hat{\i}(f)$) is minimized. 
Likewise, the term corresponding to $\hat{\e}(f)$ decreases when it is better. 
An architecture $f^* \in OPT$ will thus correspond to sup$\{\W(f, T)\colon f \in F\}$. 

"Only if" direction: 
assume $\exists f' \in F$ such that at least one of the following is true: $\p(f') < \p(f^*), \hat{\i}(f') < \hat{\i}(f^*)$, or $\hat{\e}(f') > \hat{\e}(f^*)$, and $f' \not\in OPT$. 
But that would mean that $\W(f', T) \geq \W(f^*, T)$, a contradiction. 

\end{proof}

\begin{lemma}\label{lem:wcoefficientproofepi}
Let $I = \langle f, D, W, \Xi, \Theta, \l \rangle$ be an instance of \osen, and let $T(\cdot; W^*_T; \xi_T) \in F$ be a maximum point on $\Xi$ with respect to some dataset $D$. 
An architecture $f^*(\cdot; W^*; \xi^*)  \in F$ is optimal over $\p(f),\e(f)$, and $\i(f)$, if and only if $\W(f^*, T) > \W(f, T)$ $\forall f \in F$. 
\end{lemma}
\begin{proof}
It follows immediately from \lemref{wcoefficientproof}, as well as \lemsref{ihat}{phat}{ehat}; an architecture that maximizes $\argmax{f \in F}\W(f, T)$ will belong to the Pareto optimal set for this instance of \osen. 
\end{proof}

\subsection{Algorithm}\label{sec:algorithmwritten}

Our proposed algorithm is displayed in \algref{mainalg}. We highlight two particularities of this procedure: it only evaluates a fraction $\lfloor \size{\Xi}/\epsilon \rfloor$ of the possible architectures, and it does so for a limited amount of steps, rather than until convergence. Likewise, the training algorithm remains fixed on every iteration, and is left unspecified. In \secref{analysis} we describe training procedures and optimizers under which we can guarantee optimality for a given instancce. 

\begin{algorithm}[h]
\caption{Our proposed algorithm to solve \osen. }\label{alg:mainalg}
\begin{algorithmic}[1]
   \STATE {\bfseries Input:} Architecture $f$, dataset $D$, weights space $W$, search space $\Xi$, hyperparameter set $\Theta$, interval size $1 \leq \epsilon\leq \size{\Xi}$, maximum training steps $s>0$, selected loss $\l$.
   
   \STATE Find a maximum point $T(\cdot; W^*_T; \xi^*_T)$.
   \STATE Obtain expressions for $\p(T)$ and $\hat{\i}(T)$ in terms of $\xi^*_T$. 
   \STATE $\Xi' \gets$ Sort the terms in $\Xi$ based on the leading term from $\p(T)$ \label{lst:line:sortstatement}
   \FOR{every hyperparameter set $\theta \in \Theta$} 
   \FOR{every $\epsilon^{\text{th}}$ set $\xi^{(\epsilon)} \in \Xi'$} 
   	\STATE Train a candidate architecture $f^{(i)}(\cdot; W^{(i)}; \xi^{(\epsilon)})$ for $s$ steps, under $\theta$.
   	\STATE Keep track of the largest $W$-coefficient $\W(f^{(i)}, T)$, and its corresponding sets $\xi^{(i)}$ and $W^{(i)}$ \label{lst:line:returnstatement}
   \ENDFOR
   \ENDFOR

   \RETURN $\langle \xi^*, W^* \rangle$ corresponding to the first recorded $\argmax{f^{(i)}(\cdot; W^{*}; \xi^{*})}\W(f^{(i)}, T)$
\end{algorithmic}
\end{algorithm}

\section{Analysis}\label{sec:analysis}

In this section we begin by introducing the notion of the \emph{$AB^nC$ property}, which consolidates the assumptions leading to the surrogate functions from the previous section. 
We then provide time bounds for \algref{mainalg}, and then show that, if the input presents the $AB^nC$ property, its solution runs in polynomial time. We conclude this section with an analysis of the correctness and error bounds of our algorithm, and prove under which conditions \algref{mainalg} behaves like an FPTAS.

\subsection{The $AB^nC$ Property}\label{sec:abncsec}

It is clear from \algref{mainalg} that, without any additional assumptions, the correctness proven in \lemref{wcoefficientproofepi} is not guaranteed outside of asymptotic conditions. Specifically, the work done in the previous section with the aim of expressing $\p(f), \i(f)$ and $\e(f)$ as functions on $\Xi$ implicitly induces an order on the objective functions. 
Without it, an algorithm that relies on evaluating every $\epsilon^{\text{th}}$ architecture can only provide worst case guarantees. 

The $AB^nC$ property captures sufficient and necessary conditions to guarantee such an order, as well as providing faster solutions for \algref{mainalg}. 


\begin{definition}[The weak $AB^nC$ property]\label{def:abncproperty}
Let $A, B$, and $C$ be unique layers for an architecture $f(\cdot; W; \xi) \in F$, such that, for any input $x$: 

\begin{equation}
f(x; W; \xi) = C(B_{n}(\dots B_{1}(A(x; W_A; \xi_A); W_1; \xi_1 )\dots); W_{n}; \xi_{n}); W_C; \xi_C),
\end{equation}

where $n \geq 1$ is an architectural parameter.

We say that $f(x; W; \xi)$ has the \emph{weak $AB^nC$ property}, if, for $1 \leq i \leq n$: 
\begin{align}
\p(A(\cdot; W_A; \xi_A))\text{ and } \p(C(\cdot; W_C; \xi_C)) &\in \LittleO{\p(B_i(\cdot; W_{B_i}; \xi_{B_i}))}\text{,} \label{eq:pobieq} \\
\i(A(\cdot; W_A; \xi_A))\text{ and } \i(C(\cdot; W_C; \xi_C)) &\in \LittleO{\i(B_i(\cdot; W_{B_i}; \xi_{B_i}))}. \label{eq:iobieq}
\end{align}
\end{definition}

\begin{lemma}\label{lem:allabnc}
Let $I = \langle f, D, W, \Xi, \Theta, \l \rangle$ be an instance of \osen. If there exists at least one $f \in F$ with the weak $AB^nC$ property, then all $f \in F$ have the weak $AB^nC$ property.
\end{lemma}
\begin{proof}
It follows from \defref{abncproperty}. Let $f(\cdot; W^{(f)}; \xi^{(f)}),g(\cdot; W^{(g)}; \xi^{(g)}) \in F$ be two architectures, such that $f$ has the $AB^nC$ property, and let $T$ be a maximal point on $\Xi$. 
By \lemref{phat}, it can be seen that $\p(f) \in \Theta(\p(g))$. In particular, although the value of $\p(g)$ changes under an assignment $\xi^{(g)}$, the function that defines it does not. Therefore, the condition $\p(A(\cdot; W_A; \xi_A)), \p(C(\cdot; W_C; \xi_C)) \in o( \p(B_i(\cdot; W_{B_i}; \xi_{B_i})))$, for some layers $A, B_i, C \in g$, holds. 
The proof for \eqref{iobieq} for the same layers of $g$ is symmetric, by means of \lemref{ihat}. 
\end{proof}

In spite of the closure from \lemref{allabnc}, \algref{mainalg} still requires us to find the optimal weight set $W^*$ for every architecture evaluated. Given that the set $W$ is oftentimes large, the algorithm itself does not adjust to modern deep learning practices. Therefore we strengthen \defref{abncproperty} by imposing some further conditions on the metric function $\l$. 

\begin{definition}[The strong $AB^nC$ property]\label{def:strongabncproperty}
Let $I = \langle f, D, W, \Xi, \Theta, \l \rangle$ be an instance of \osen. We say $I$ has the \emph{strong $AB^nC$ property} if $\exists f \in F$ with the weak $AB^nC$ property and $\nabla_W\l$ is $L$-Lipschitz smooth with respect to $W$, with bounded stochastic gradients, for every $\theta \in \Theta$. 
That is:


\begin{equation}
\size{\size{\nabla \l(f(x_i; \cdot), y_i) - \nabla \l(f(x_j; \cdot), y_j)}} \leq  L\size{\size{\l(f(x_i; \cdot), y_i) - \l(f(x_j; \cdot), y_j)}},
\end{equation}
and
\begin{equation}
\size{\size{\nabla \l(f(x_i; \cdot), y_i)_k}} \leq  G\quad\quad \forall k \in [\size{B}], 
\end{equation}

$\forall \langle x_i, y_i \rangle, \langle x_j, y_j \rangle \in B \subset D$, and some fixed constants $L$ and $G$. 



\end{definition}

Remark that the constraints from \defref{strongabncproperty} are pretty common assumptions the analysis of first-order optimizers in machine learning \cite{ShalevShwartz2009StochasticCO,Zaheer2018AdaptiveMF}. 
While there is no algorithm to find an $\epsilon$-stationary point in the non-convex, non-smooth setting\footnote{This limitation, as argued by \newcite{LeCun}, in addition to the large track record of successes by deep learning systems, has little impact on natural problems.},  approximating a point near it is indeed tractable \cite{ZhangAndLinEtAl}. 

We can exploit such a result with the goal of showing that the set $F$ has an ordering with respect to $\p(f), \hat{\i}(f)$, and $\hat{\e}(f)$, through the expected convergence of SGD in the non-convex case. Bounds on this convergence are well-known in the literature (see, for example, \newcite{ghadimi2013stochastic}). 
However, \defref{strongabncproperty} has looser assumptions, which align directly with a recent result by \newcite{UnifiedNguyen} on a variant of SGD referred to there as \emph{shuffling-type}, which we reproduce below for convenience:

\begin{lemma}[\newcite{UnifiedNguyen}]\label{lem:nguyenslemma}
Let $f(x; w)$ be a $L$-Lipschitz smooth function on $w$, bounded from below and such that $\size{\size{\nabla f(x_i; w)}} \leq G$ for some $G$ and $i \in [n]$. Given the following problem:
\begin{align}
 \mmin{w} F(w) &= \frac{1}{n}\sum_{i}^n f(x_i; w),
\end{align}

the number of gradient evaluations from a shuffling-type SGD required to obtain a solution $\hat{w}_T$ bounded by

\begin{equation}
\mathbb{E}[\vert\vert\nabla F(\hat{w}_T)\vert\vert^2] \leq \epsilon
\end{equation}

is given by:

\begin{equation}
T = \left\lfloor 3LG[F(\tilde{w}_0) - \inf{w} F(w)] \cdot \frac{n}{\epsilon^{3/2}}\right\rfloor.
\end{equation}

for a fixed learning rate $\eta = \frac{\sqrt{\epsilon}}{LG}$. 






\end{lemma}

\begin{lemma}\label{lem:abncruntime}
Let $I = \langle f, D, W, \Xi, \Theta, \l \rangle$ be an instance of \ose with the strong $AB^nC$ property. 
Fix a maximum point $T$ and a training procedure with a shuffling-type SGD optimizer. 
Then, for a fixed number of steps $s$ and fixed $\theta \in \Theta$, for any $f,g \in F$, if their depths $n_f \in \xi_f, n_g \in \xi_g$ are such that $n_f \leq n_g$, then $\hat{\e}(f) \leq \hat{\e}(g)$ .


\end{lemma}
\begin{proof}

Let $s_f$, (resp. $s_g$) be the total number of steps required to achieve an $\epsilon$-accurate solution for $f$ (resp. $g$). Then, by \lemref{nguyenslemma}, and counting the actual number of operations on the network:

\begin{align}
s_f &\leq \frac{C}{\epsilon^{3/2}}\i(f) \label{eq:sff} \\
s_g &\leq \frac{C}{\epsilon^{3/2}}\i(g) \label{eq:sgg}
\end{align}

where $C$ is a constant equal in both cases due to the parameters of the training procedure. By the weak $AB^nC$ property, $s_f < s_g$, since $n_f < n_g \implies \i(f) \leq \i(g)$. Thus, for a fixed step $s$:

\begin{equation}
\frac{C}{{\epsilon_f}^{3/2}}\i(f) = \frac{C}{{\epsilon_g}^{3/2}}\i(g),
\end{equation}

which immediately implies that ${\epsilon_g} \geq {\epsilon_f}$, w.p.1, with equality guaranteed when $\i(g) = \i(f)$. Since $\epsilon_g$ (resp. $\epsilon_f$) is the radius of a ball centered around the optimal point, then $\hat{\e}(f) \leq \hat{\e}(g)$. 

\end{proof}

\begin{remark}
If we were to change the conditions from \lemref{abncruntime} to any loss with convergence guarantees (e.g., for $\mu$-strongly convex $\l(f(x_i; \cdot), y_i)$), the statement would still hold.
\end{remark}

With these tools, we can now state the following lemma:

\begin{lemma}\label{lem:abncorder}
Let $I = \langle f, D, W, \Xi, \Theta, \l \rangle$ be an instance of \osen, such that $I$ presents the strong $AB^nC$ property. Let $T$ be a maximum point on $\Xi$. 
Then, for any $f,g \in F$:
\begin{align}
\p(f) \preceq \p(g) &\iff \hat{\i}(f) \preceq \hat{\i}(g) \label{eq:pprecp} \\
\p(f) \preceq \p(g) &\iff \hat{\e}(f) \preceq \hat{\e}(g) \label{eq:eprece}
\end{align}
\end{lemma}
\begin{proof}
The proof for \eqref{pprecp} is immediate from the definition of the weak $AB^nC$ property: since the inference speed as well as the parameter size are bounded by the depth of the network, there exists a partial ordering on the architectures with respect to the original set. 
\eqref{eprece} follows immediately from \lemref{abncruntime} and the definition of the strong $AB^nC$ property. 
\end{proof}




\subsection{Time Complexity}\label{sec:timecomplexity}

In this section we prove the time complexity of \algref{mainalg}.



\begin{theorem}\label{thm:maintheorem}
Let $I = \langle f, D, W, \Xi, \Theta, \l \rangle$ be an instance of \ose with the $AB^nC$ property. 
Then if \algref{mainalg} employs a shuffling-type SGD training procedure, it terminates in
\begin{equation}
\BigO{\size{\Xi} + \size{W^*_T}\left(1 +  \size{\Theta}\size{B}\size{\Xi}\frac{1}{\epsilon \cdot s^{3/2}}\right)}
\end{equation}
steps, where $1 \leq \epsilon \leq \size{\Xi}$, and $s > 0$ are input parameters; $B \subset D$ for $B$ in every $\theta^{(i)} \in \Theta$; and $\size{W^*_T} = \argmax{w \in W}(\size{w})$ is the cardinality of the largest weight set assignment.

\end{theorem}
\begin{proof}

By \lemref{nguyenslemma} and \lemref{abncruntime}, finding the maximum point $T(\cdot; W^*_T; \xi_T)$ will be bounded by $\size{\Xi}$. 
Likewise, obtaining the expressions for $\p(T)$ and $\hat{\i}(T)$ can be done by traversing $T$ and employing a counting argument. 
By the definition of the weak $AB^nC$ property and \lemref{abncorder}, this is bounded by the cardinality of its weight set. 
On the same vein, grouping the terms of $\p(T)$ can be done in constant time by relating it to the depth of the network. 

Then the initialization time of our algorithm will be:

\begin{equation}
\BigO{\size{\Xi} + \size{W^*_T}}.
\end{equation}

On the other hand, training each of the $\size{\Xi}$ candidate architectures with an interval size $\epsilon$ takes

\begin{equation}
\BigO{\left\lfloor\frac{\size{\Xi}}{\epsilon}\right\rfloor \i(f)\frac{\size{B}\size{\Theta}}{s^{3/2}}}
\end{equation}

steps. 
Adding both equations together, and by \defref{maxpoint}, this leads to a total time complexity of:

\begin{equation}
\BigO{\size{\Xi} + \size{W^*_T}\left(1 +  \size{\Theta}\size{B}\size{\Xi}\frac{1}{\epsilon \cdot s^{3/2}}\right)}
\end{equation}

steps, which concludes the proof. 
\end{proof}

So far we have assumed training is normally done through the minimization of the surrogate loss $\hat{\e}(f)$. Other methods, such as Bayesian optimization, simulated annealing, or even random search, can also be used to assign weights to $f$. 
In the general case, when the algorithm is guaranteed to return the optimal set of weights regardless of the convexity and differentiability of the input function, any reasonable procedure would take at most $O(\i(f)\size{D}\size{W})$ steps to obtain such a solution.

As shown in \thmref{fptasosethm}, when the input does not present the $AB^nC$ property, the runtime is not necessarily polynomial. Likewise, its runtime might be worse than a simple linear search when $\size{W} << \size{\Xi}$. 
This last situation, however, is only seen in zero-shot techniques.

\subsection{Error Bounds}\label{sec:approxratio}

In this section we prove error bounds for \algref{mainalg} under the strong $AB^nC$ property. The general case, where it is impossible to assume whether \lemref{abncorder} holds, can only provide asymptotically worst-case guarantees.



\begin{theorem}\label{thm:approximationratiothm}
Let $I = \langle f, D, W, \Xi, \Theta, \l \rangle$ be an instance of \ose with the strong $AB^nC$ property. For a chosen $s > 0$ and $1 \leq \epsilon \leq \size{\Xi}$, if \algref{mainalg} employs a shuffling-type SGD training procedure with a fixed learning rate, it will return a solution with a worst-case absolute error bound of $c \leq \epsilon - 1$ on the space of solutions reachable by this training procedure. 

\end{theorem}
\begin{proof}

The proof has two parts: first, proving that the worst-case absolute error bound for an $s$-optimal solution is at most $\epsilon - 1$; second, showing that $c$ is independent of the choice of $s$, and that these results extend to an optimal solution found by taking every architecture to convergence with the optimizer.

\paragraph{} To begin, assume for simplicity $\Theta = \{\theta\}$. 

By \lemstworef{wcoefficientproof}{wcoefficientproofepi} we know that, for a fixed $s$, $\W(f^{s,*}, T)$ is $s$-optimal in the limit where $\epsilon = \size{\Xi}$. Let $OPT_{s}$ be the location of the solution in $\Xi$-space to such an instance, and let $OPT_{ALG}$ be the corresponding location of the output of \algref{mainalg}. 

The input instance $I$, by \lemref{abncruntime}, contains an ordering for the error with respect to the parameter size. Sorting the parameters and then obtaining a solution $OPT_{ALG}$ means that $OPT_s$ is located within the radius of the ball centered at $\W(f^{s}, T)$'s index, that is, $OPT_{s} \in \mathbf{B}(OPT_{ALG}, r)$. 
It is not difficult to see that $r = \epsilon - 1$, and 
therefore
\begin{align}
OPT_{s} &\geq OPT_{ALG} - (\epsilon-1), \text{ or } \\
OPT_{s} &\leq OPT_{ALG} + (\epsilon-1).
\end{align}

We can simplify the above expression to

\begin{equation}\label{eq:approxratio1}
\vert OPT_{s} - OPT_{ALG} \vert \leq \epsilon - 1,
\end{equation}

which yields the desired absolute error bound. 

Now let $\size{\Theta}> 1$. By line~\ref{lst:line:returnstatement}, \eqref{approxratio1} still holds, since an optimal solution $\W_{i}$ over $\theta^{(i)} \in \Theta$ will be returned iff $\W_i(f^{s}, T) \geq \W_j(f^{s}, T)$ for all $\theta^{(i)}, \theta^{(j)} \in \Theta$.

\paragraph{}
Consider the case $s+1$. Given that the instance has the $AB^nC$ property, \lemref{abncorder} holds, and following \eqref{approxratio1} we obtain

\begin{equation}
\vert OPT_{s+1} - OPT_{ALG} \vert \leq \epsilon - 1.
\end{equation}

By an induction step on some $s_T>s$, we have:

\begin{equation}\label{eq:approxratio3}
\vert OPT_{s_T} - OPT_{ALG} \vert \leq \epsilon - 1,
\end{equation}

from which immediately follows that $c \leq \epsilon - 1$ for any choice of $s$. 

To finalize our proof, remark that $OPT_{s_T}$ is only $s_T$-optimal. 
Consider then the case when for the optimal solution $OPT_{*}$--that is, it is the solution corresponding to the smallest number of steps required to converge to an optimal solution $W^* \in W$ for all $f^*$ and $\theta_i \in \Theta$. By the definition of the strong $AB^nC$ property, \lemref{nguyenslemma}, and the induction step from the last part:

\begin{equation}
\vert OPT_{*} - OPT_{ALG}\vert \leq \epsilon - 1,
\end{equation}

which concludes the proof.

\end{proof}

 While \thmref{approximationratiothm} provides stronger results, we also provide the approximation ratio of \algref{mainalg}, to be used later.

\begin{corollary}\label{cor:strongabncapproxratio}
If the instance presents the strong $AB^nC$ property, \algref{mainalg} returns a solution with an approximation ratio 
\begin{equation}
\rho \leq \size{1 - \epsilon}
\end{equation}
 on the space of solutions reachable by this training procedure. 
\end{corollary}
\begin{proof}
Follows immediately from \thmref{approximationratiothm}, and the definition of an approximation ratio:
\begin{align}
\rho &\leq \left\vert\frac{OPT_* - OPT_{ALG}}{OPT_*}\right\vert \\
\rho &\leq \left\vert\frac{(k - (\epsilon - 1))\left\lfloor\frac{\size{\Xi}}{\epsilon}\right\rfloor - k\left\lfloor\frac{\size{\Xi}}{\epsilon}\right\rfloor}
{(k - (\epsilon - 1))\left\lfloor\frac{\size{\Xi}}{\epsilon}\right\rfloor}\right\vert
\end{align}

for some $k \leq \epsilon$. Plugging it in we obtain the desired ratio. 
\end{proof}

\begin{remark}
Note that $OPT_{*}$ does not necessarily correspond to a global optimum, but it corresponds to the optimum reachable under the combination of the given $\Theta$, as well as the chosen optimization algorithm. 
\end{remark}

The previous results show that an input with the $AB^nC$ property is well-behaved under \algref{mainalg}. Specifically, they show dependence on their time and error bounds with respect to the input parameters $s$ and $\epsilon$. \thmref{fptasosethm} suggests that this problem admits an FPTAS, and \thmref{fptasthm} crystallizes this observation for a subproblem of \osen. 

\begin{theorem}\label{thm:fptasthm}
Let \textsc{ose-strong-abnc} be a subproblem of \ose that only admits instances with the strong $AB^nC$ property. Assume that $\hat{\e}(f; W; \xi)$ is $\mu$-strongly convex for all instances. 
Then \algref{mainalg} is an FPTAS for \textsc{ose-strong-abnc}, for any choice of optimizer. 
\end{theorem}
\begin{proof}
By \thmref{maintheorem}, the running time of our algorithm is polynomial on both the size of the input, and on 
$1/\epsilon$. 
From \corref{strongabncapproxratio} we obtain the desired approximation ratio. 
Given the convexity assumptions, the approximation ratio applies to the global optimal solution which would be obtained via an exhaustive search. 
By definition, this algorithm is an FPTAS for \textsc{ose-strong-abnc}.
\end{proof}

\begin{remark}
By \thmref{nnhard}, lifting the convexity condition from \thmref{fptasthm} would not guarantee the same global optimization guarantees. However, if we are to consider solely the set of instances reachable by the optimizer's algorithm as the optimal solution space, the results from \thmref{fptasthm} still hold without the need to assume $\mu$-strong convexity.
\end{remark}



\section{Conclusion}\label{sec:conclusion}

We presented a formal analysis of the OSE problem, as well as an FPTAS that works for a large class of neural network architectures. 
The time complexity for \algref{mainalg} can be further improved by using a strongly convex loss \cite{RakhlinEtAl} or by employing optimizers other than SGD. A popular choice of the latter is ADAM \cite{Adam}, which is known to yield faster convergence guarantees for some convex optimization problems, but not necessarily all \cite{Reddy}. 
Regardless, it is also well-known that SGD finds solutions that are highly generalizable \cite{Kleinberg,DinhEtAl}.

On the same vein, tighter approximation bounds can be achieved even for cases where only the weak $AB^nC$ property is present, as one could induce an ordering on the error by making it a convex combination of the candidate's surrogate error and of the maximum point's. This, however, implies that one must train the maximum point to convergence, impacting the runtime of the algorithm by a non-negligible factor which may or may not be recovered in an amortized context. 

We began this paper by indicating that obtaining a smaller network with the OSE objective would be more efficient in terms of speed, size, and environmental impact; yet, the reliance of our algorithm on a maximum point on $\Xi$ appears to contradict that. 
Such a model is, by design, an analysis tool, and does not need to be trained. 
The optimality of the $W$-coefficient, as well as the other results from this paper, still hold without it. 
Finally, the keen reader would have noticed that we alternated between the computational complexity analysis of the hardness of training a neural network, and its statistical counterpart; we were able to bridge such a gap by treating the training process as a black-box and imposing simple convergence constraints. 
A more in-depth analysis of the differences between both approaches, as well as further hardness results on learning half-spaces, is done by \newcite{DanielyEtAl}. 
This technique proved itself helpful in yielding the results around the convergence and time complexity guarantees of the work on our paper. We hope that it can be further applied to other deep learning problems to obtain algorithms with well-understood performances and limitations. 

\section*{Acknowledgments}
The author would like to thank D. Perry for his detailed questions about this work, which ultimately led to a more generalizable approach. 
Special thanks are in order for B. d'Iverno, and Q. Wang, whose helpful conversations about the approximation ratio were crucial on the analysis.

\bibliography{biblio}

\appendix
\section*{Appendices}

\section{Proof of \lemref{npcose}}\label{app:npcoseproof}

We can prove \lemref{npcose} by leveraging previous results from the literature. Our--admittedly simple--argument relies on noting that training a neural network is equivalent to our problem when $\Xi = \{\xi\}$. 
Said problem is known to be computationally hard in both its decision \cite{BartlettBenDavid} and optimization \cite{ShalevUML} versions, and has been formulated often.\footnote{Aside from the plethora of negative results mentioned in \secref{relatedwork}, \newcite{Livni} provides comprehensive results under which circumstances it is indeed possible to efficiently train a neural network.} 
We will refer to its decision version as \textsc{nn-training-dec}, and, for convenience, we reproduce it below with our notation:

\begin{definition}[\textsc{nn-training-dec}]\label{def:nndec}
Given a dataset sampled i.i.d. from an unknown probability distribution $D = \{\langle x_i, y_i \rangle\}_{i=1,\dots,m}$, a set of possible weight assignments $W$, an architecture $f(\cdot; \cdot; \xi)$, and a number $k$, the \textsc{nn-training-dec} problem requires finding a valid combination of parameters $W^* \in W$ such that
\begin{equation}
\e(f(\cdot; W^*; \xi), D) \leq k.
\end{equation}
\end{definition}

\begin{theorem}\label{thm:nnhard}
For arbitrary $W, \xi, D$, and $k$, \textsc{nn-training-dec} $\in$ \NP-hard.
\end{theorem}
\begin{proof}
Refer to \secref{relatedwork} for proofs of multiple instances of this problem. 
\end{proof}

We can now prove \lemref{npcose}:

\begin{proof}
It can be seen that \ose is in \NP, as  \defref{osedec} has as a set of certificates all possible tuples $\langle \xi^{(i)} \in \Xi, W^{(j)} \in W \rangle$. 
Any certificate can be evaluated in a polynomial number of steps by calculating the three surrogate functions on the given configuration, with an added dependency on $D$ due to the computation of $\e(f)$. 

With that, we can now show a trivial, linear-time reduction from \textsc{nn-training-dec} to \textsc{ose-dec}:

Let $I_{nn} = \langle f, D, W, k \rangle$ be an instance of \textsc{nn-training-dec}. 
We can augment $I_{nn}$ by adding the architectural parameters of $f$, $I_{nn}' = I_{nn} \cup \langle \xi \rangle$. Since they are not present in $I_{nn}$, they must be obtained by traversing the network. This can be done in time linear in the size of $f$.  

Then, an equivalent input to \textsc{ose-dec} would be 
\begin{equation}
I_{ose} = \langle f'=f, D'=D, W'=W, \Xi =\{\xi\}, \Theta = \{\}, k_p = \infty, k_i = \infty, k_e = k \rangle.
\end{equation}

Solving \textsc{ose-dec} will return a network $f'(\cdot; W^*; \xi)$ for which $\p(f'(\cdot; W^*; \xi))$, $\i(f'(\cdot; W^*; \xi))$ and $\e(f'(\cdot; W^*; \xi), D')$ are $k$-minimal across all $\xi \in \Xi$, $W^* \in W'$, since $\Xi = \{\xi\}$ and $W = W'$. 

It follows that \textsc{nn-training-dec} $\leq_p$ \textsc{ose-dec}, and by \thmref{nnhard} and the above, we have then that \textsc{ose-dec} $\in$ \NP-hard. This concludes the proof of \lemref{npcose}.
\end{proof}

\end{document}